\documentclass{article}

\usepackage{microtype}
\usepackage{graphicx}
\usepackage{subcaption}
\usepackage{booktabs} %

\usepackage{hyperref}

\usepackage[preprint]{icml2026}

\usepackage{amsmath}
\usepackage{amssymb}
\usepackage{mathtools}
\usepackage{amsthm}

\usepackage[capitalize,noabbrev]{cleveref}

\theoremstyle{plain}
\newtheorem{theorem}{Theorem}[section]

\theoremstyle{definition}
\newtheorem{definition}[theorem]{Definition}

\theoremstyle{remark}

\newtheorem{example}[theorem]{Example}
\newtheorem{problem}[theorem]{Problem}

\usepackage[textsize=tiny]{todonotes}
\usepackage{templates-papers}
\usepackage{wrapfig}
\usepackage{paralist}
\usepackage{flushend}
\usepackage{makecell}

\icmltitlerunning{On Improving Neurosymbolic Learning by Exploiting the Representation Space}

\begin{document}

\twocolumn[
  \icmltitle{On Improving Neurosymbolic Learning by Exploiting the Representation Space}

  \icmlsetsymbol{equal}{*}
  \icmlsetsymbol{efi}{+}

  \begin{icmlauthorlist}
    \icmlauthor{Aaditya Naik}{equal,upenn}
    \icmlauthor{Efthymia Tsamoura}{equal,efi,huawei}
    \icmlauthor{Shibo Jin}{upenn}
    \icmlauthor{Mayur Naik}{upenn}
    \icmlauthor{Dan Roth}{oracle,upenn}
  \end{icmlauthorlist}

  \icmlaffiliation{upenn}{University of Pennsylvania}
  \icmlaffiliation{huawei}{Huawei Labs}
  \icmlaffiliation{oracle}{Oracle AI}

  \icmlcorrespondingauthor{Aaditya Naik}{asnaik@seas.upenn.edu}

  \icmlkeywords{Machine Learning, ICML}

  \vskip 0.3in
]

\printAffiliationsAndNotice{\icmlEqualContribution{} \efiContribution{}}  %

\begin{abstract} 
We study the problem of learning neural classifiers in a neurosymbolic setting where the hidden gold labels of input instances must satisfy a logical formula. 
Learning in this setting proceeds by first computing (a subset of) the possible combinations of labels that satisfy the formula and then computing a loss using those combinations and the classifiers' scores.              
One challenge is that the space of label combinations can grow exponentially, making learning difficult.
We propose a technique that prunes this space by exploiting the intuition that instances with similar latent representations are likely to share the same label. 
While this intuition has been widely used in weakly supervised learning, its application in our setting is challenging due to label dependencies imposed by logical constraints. 
We formulate the pruning process as an integer linear program that discards inconsistent label combinations while respecting logical structure.
Our approach, \ours{}, is orthogonal to existing training algorithms and can be seamlessly integrated with them.
Across 16 benchmarks over complex neurosymbolic tasks, we demonstrate that \ours{} boosts the performance of state-of-the-art neurosymbolic engines like \scallop, \dolphin, and \upenn 
by up to 48\%, 53\%, and 8\%, leading to state-of-the-art accuracies.

\end{abstract}

\section{Introduction}\label{section:introduction}

\textbf{Motivation.} \emph{Neurosymbolic learning} (NSL), i.e., the integration of symbolic with neural mechanisms for inference and learning, has been proposed as the remedy for some of the most vulnerable aspects of deep networks \cite{feldstein2024mappingneurosymbolicailandscape}. Recent works have shown that NSL holds immense promise, offering, in addition, the means to train neural networks using weak labels \cite{concordia,NeurIPS2023}. 
We study the problem of learning neural classifiers in frameworks where a symbolic component ``sits'' on top of one or more neural classifiers and learning is weakly supervised  \cite{deepproblog-journal}. An example of our setting, referred to as \nesy, is presented below.    

\begin{example} [\nesy~example] \label{example:NSL}
    Consider the classical example of learning an MNIST classifier $f$ using training samples of the form ${(\{x_1,x_2\}, \phi)}$, where $x_1$ and $x_2$ are MNIST digits and $\phi$ is a logical sentence that the gold labels of $x_1$ and $x_2$, $l_1$ and $l_2$, should satisfy \cite{deepproblog-journal}. Unlike supervised learning, $l_1$ and $l_2$ are unknown to the learner. The logical sentence $\phi$ restricts the space of labels that can be assigned to $x_1$ and $x_2$. For example, consider the training sample ${(\{x_1,x_2\},\phi_1 := l_1 + l_2 = 8)}$. According to this sample, any combination of $l_1$ and $l_2$ whose sum is 8 is valid and all other combinations are invalid, e.g., $l_1 = 2$ and $l_2 = 6$ is valid, but $l_1 = 3$ and $l_2 = 6$ is invalid.  
    In total, there are $9$ different combinations of $l_1$ and $l_2$ that satisfy $\phi_1$. 
    The gold labels of $x_1$ and $x_2$ are 1 and 7, respectively. However, they are unknown during learning.
\end{example}

\nesy~is a popular frameworks in the NSL literature, with DeepProbLog \cite{deepproblog-journal}, NeuroLog \cite{neurolog}, \scallop~\cite{scallop}, \dolphin~\cite{Dolphin}, and \upenn~\cite{ised} being a few frameworks that rely on it. 
In addition, as discussed in \cite{NeurIPS2023}, \nesy~encompasses \emph{partial label learning} (PLL) \cite{cour2011,structured-prediction-pll}, where each input instance is associated with mutually exclusive candidate labels, and
learning classifiers subject to constraints on their outputs
\cite{relaxed-supervision,Aggregate-observations}.
\nesy{} has wide range of applications, including
fine-tuning language models \cite{scallop-llms}, aligning video to text \cite{laser}, visual question answering \cite{scallop}, and learning knowledge graph embeddings \cite{maene2025}.   

\textbf{Limitations.}
Learning in \nesy~proceeds by first computing (a subset of) the possible combinations of labels that lead to the given learning target, subject to the symbolic component, and then computing a loss using those combinations and the classifiers' scores.
However, learning is challenging when the space of possible label combinations is large
\cite{reasoning-shortcuts,wang2024imbalances}. This is because supervision becomes weaker. The question arises: \emph{Are there circumstances where we can safely discard specific label combinations?} 

\textbf{Contributions.} We propose a {plug-and-play} technique to reduce the space of candidate label combinations by \emph{exploiting the inconsistency between the representation space and the space of candidate label combinations} {for \nesy{} engines.} 
Our intuition is that if the latent representations of two instances are very close, then they belong to the same class, and hence share the same gold labels. 
When applied to \nesy, this intuition can substantially reduce candidate label combinations during training:  
\begin{example} \label{example:NSL2} [Contd Example~\ref{example:NSL}] 
    Consider a second training sample $(\{x'_1,x'_2\}, \phi_2 := l'_1 + l'_2 = 2)$, where $l'_1$ and $l'_2$ correspond to the gold labels of $x'_1$ and $x'_2$. According to this sample, the valid combinations of labels for $(x'_1, x'_2)$ are $(0,2)$, $(1,1)$, and $(2,0)$. If the latent representations of $x_1$ and $x'_1$ are very close, then $l_1$ must range in $\{0,1,2\}$. Hence, the number of candidate label combinations associated with the first training sample reduces from $9$ to $3$. 
\end{example}
The PLL literature has extensively investigated techniques that take advantage of the representation space to eliminate erroneous candidate labels during training \cite{consistent-pll,papi-cvpr2023,pico,instance-dependent-pll}. In fact, the intuition in the above example has been successfully adopted in weakly supervised learning \cite{he2024candidate}. However, its straightforward adoption in \nesy~is problematic, as it can result in training samples associated with zero supervision, i.e., without candidate combinations of labels.  
Recent \nesy{} literature has also explored integration with representation learning \cite{vael,gedi}, however, the works do not focus on pruning the candidate combinations of labels.  
To our knowledge, \ablsim{} is the only \nesy{} technique that shares the same objectives as our work \cite{ABLSim}. \ablsim{} extends \abl{} \cite{ABL}
with a strategy that uses the similarity between samples to guide the search away from erroneous candidate labels. 
However, \ablsim{} is prone to local optima due to its dependence on \abl's search strategy, see \cite{neurolog,NeurIPS2023} for a discussion. 
As our empirical analysis shows, \ablsim{} can lead to lower classification accuracy than state-of-the-art \nesy{} engines that do not exploit the representation space. The above manifests that integrating ideas from representation learning with ideas from \nesy{} learning is not trivial.

In this work, we propose a technique that organizes the training samples and their associated candidate label combinations into a graph, called the \textit{proximity graph}. The edges in the graph reflect the proximity of instances in the representation space. Then, by generalizing the intuition in our example, we introduce the problem of discarding the maximum number of candidate label combinations subject to the edges in the graph under the constraint that each training sample is associated with at least one candidate label combination. 
We then propose a solution to this problem by casting it into an \emph{integer linear program} (ILP). 

Our approach offers two unique benefits. First, it is complementary to \nesy~training algorithms: Our technique first discards candidate label combinations; then training proceeds with the remaining label combinations. 
Second, it can be employed in a training-free manner, i.e., we can discard candidate label combinations using a pre-trained encoder, 
such as a large vision and language model \cite{Blip2}, or ResNet \cite{resnet}, before training. 
Alternatively, it may be applied during training, i.e., by using the encoder trained so far to extract features for the corresponding training instances, then training it with the label combinations that have not been discarded, and repeating the process. 

We evaluate the benefits of our technique, called \ours, applying it in combination with three state-of-the-art neurosymbolic engines, \scallop, \dolphin, and \upenn, on a variety of benchmarks that range from digit classification -- the classic SUM-$M$, MAX-$M$, and HWF-$M$ benchmarks \cite{deepproblog-journal} -- to visual question answering and video-to-text alignment. The integration of \ours~with the above engines was rather straightforward: we employed \ours~to filter out pre-images during the pre-image computation phase and then used the remaining pre-images to train the classifier.
\ours{} improves the accuracy across all engines and benchmarks. In our most challenging benchmarks, MAX-40 and MUGEN, the baseline accuracy improves by up to $\sim 12\%$ and $\sim 53\%$, respectively. 
For \scallop, \dolphin, and \upenn, the maximum accuracy improvements are $\sim 48\%$, $\sim 53\%$, and $\sim 8\%$, respectively.      
In addition, \ours{} (1) 
is robust to the underlying encoder, (2) maintains with high probability the gold pre-images even when the encoder is randomly initialized, and (3) incurs a small runtime overhead.   
Our contributions are:

\begin{compactitem}%
    \item We formalize the problem of discarding label combinations for a set of \nesy~training samples based on the proximity of the latent representations of their instances.
    \item We propose an ILP-based algorithm that guarantees that each training sample retains at least one candidate label combination while maximizing the number of discarded label combinations.
    \item We evaluate our technique with different neurosymbolic engines and benchmarks and demonstrate improvements in classification accuracy of up to 53\%. 
\end{compactitem}

\section{Preliminaries}\label{section:preliminaries}

Our notation is summarized in Table~\ref{table:notation} in the appendix.
\textbf{Supervised learning.} 
For an integer $n \ge 1$, let $[n] := \{1,\dots,n\}$. 
Let also $\mathcal{X}$ be the instance space and $\mathcal{Y} = [c]$ be the output space.
We use $x,y$ to denote elements in $\+X$ and $\+Y$.
We consider \emph{scoring functions} of the form ${f: \+X \mapsto \Delta_c}$, where $\Delta_c$ is the space of probability distributions over $\+Y$, e.g., $f$ outputs the softmax probabilities (or \emph{scores}) of a classifier. 
We use ${f^j(x)}$ to denote the score of ${f(x)}$ for class $j \in \+Y$. A scoring function $f$ induces a \emph{classifier} $[f]: \mathcal{X} \mapsto \mathcal{Y}$, whose \emph{prediction} on $x$ is given by $ \argmax_{j \in [c]} f^j(x)$. 
The aim us to learn $f$ using samples of the form ${(x,y)}$. 

\textbf{Neurosymbolic learning.} 
We assume familiarity with basic notions of logic, such as the notions of variables, constants, predicates, facts, rules, and sentences. We use small for constants and predicates, and capitals for variables. 
To ease the presentation, we assume a single classifier $f: \+X \rightarrow \+Y$. However, our results straightforwardly extend to settings with multiple classifiers.  
Let $\+K$ be a background logical theory. As mentioned in Section~\ref{section:introduction}, $\+K$ reasons over the predictions of ${f}$. Of course, this is possible by translating neural predictions into facts, e.g., returning to Example~\ref{example:NSL}, DeepProbLog, \scallop, and \dolphin, create one fact of the form 
$digit(d,x_1)$ for each possible digit $d$ and associate this fact with the softmax score of class $d$ for $x_1$ (and similarly for $x_2)$. 
If $\+K$ is propositional, we can translate the outputs of $f$ into propositions \cite{pmlr-v80-xu18h}.
Reasoning over the resulting facts using $\+K$ produces the overall output.
Different frameworks may employ different reasoning semantics at testing time which is orthogonal to this work. 

Unlike supervised learning, in \nesy, each training sample is of the form $(\_x,\phi)$, where $\_x$ is a set of elements from $\+X$ and $\phi$ is a logical sentence (or a single target fact in the simplest scenario).
The gold labels of the input instances may be unknown to the learner. Instead, we may only know that the gold labels of the elements in $\_x$ satisfy the logical sentence $\phi$ subject to $\+K$. 
In Example~\ref{example:NSL}, $\+K$ is empty. However, in one of the benchmarks that we consider in our experiments, namely VQAR \cite{scallop}, $\+K$ is commonsense knowledge from CRIC \cite{CRIC}. 

Formula $\phi$ and $\+K$ allow us to ``guess'' what the gold labels of the elements in $\_x$ might be so that $\phi$ is logically satisfied subject to $\+K$. This is essentially the process of \textit{abduction} \cite{neurolog}. To align with the terminology in \cite{NeurIPS2023}, for a training sample $(\_x,\phi)$, a \textit{pre-image}\footnote{Pre-images correspond to \textit{proofs} in \cite{neurolog,scallop,deepproblog-journal}.}
is a combination of labels of the elements in $\_x$, such that 
$\phi$ is logically satisfied subject to $\+K$. 
The gold pre-image is the one mapping each instance to its gold label. 
By construction, each \nesy~training sample includes the gold pre-image.  
Abduction allows us to ``get rid of'' $\phi$ and $\+K$ and represent each training sample via $\_x$ and its corresponding pre-images, i.e., 
as $(\_x, \{\sigma_{i}\}_{i=1}^{\omega} )$, where each pre-image $\sigma_{i}$ is a mapping from $\_x$ into $\+Y$. 
We use $\+D = \{(\_x_\ell, \{\sigma_{\ell,i}\}_{i=1}^{\omega_\ell} )\}_{\ell=1}^n$ to denote a set of $n$ training samples.

\begin{example} \label{example:NSL3} [Contd Example~\ref{example:NSL2}] 
    Candidate pre-images for the first sample are:  
    $\sigma_{1,1} = \{x_1 \mapsto 0, x_2 \mapsto 8 \}$,
    $\sigma_{1,2} = \{x_1 \mapsto 1, x_2 \mapsto 7 \}$, and 
    $\sigma_{1,3} = \{x_1 \mapsto 8, x_2 \mapsto 0 \}$.
    Two candidate pre-images of the second sample are: 
    $\sigma_{2,1} = \{x'_1 \mapsto 0, x'_2 \mapsto 2\}$ and
    $\sigma_{2,2} = \{x'_1 \mapsto 1, x'_2 \mapsto 1\}$.
\end{example} 
Our notation of pre-images is equivalent to the notation of training samples in \cite{NeurIPS2023}. 
The only thing left to discuss is what is the learning objective in \nesy. 
Each \nesy~framework adopts its own learning objective. For example, in DeepProbLog and \scallop, the aim is to minimize semantic loss \cite{pmlr-v80-xu18h} or its approximations \cite{scallop}. The authors in
\cite{NeurIPS2023} formalize learning via minimizing \emph{zero-one partial loss}, that is the probability $\phi$ not being logically satisfied subject to $\+K$.  
Our work is orthogonal to the actual loss used for training. 

\section{Discarding Pre-Images Based on Latent Representations}\label{section:contributions}

We aim to reduce the number of candidate pre-images of the \nesy~training samples by exploiting inconsistencies with the representation space. 
The question naturally arises: \emph{Can a reduction in the number of pre-images per training sample lead to classifiers with higher accuracy?} 
The NSL community has verified this claim both experimentally \cite{neurolog,scallop} and theoretically \cite{reasoning-shortcuts,wang2024imbalances}. 
For example, \cite{reasoning-shortcuts} showed that the number of deterministic classifiers that minimize semantic loss \cite{pmlr-v80-xu18h} is directly proportional to the number of abductive proofs per training sample (i.e., pre-images in our terminology), while \cite{wang2024imbalances} showed that the probability a classifier misclassifies instances of the given class is a direct function of the number of pre-images.

Central to our technique are two notions: \emph{proximity graphs} and \emph{consistency}. Proximity graphs are graphs whose edges reflect the proximity of latent instance representations. As we will see later, proximity between instances imposes restrictions on the pre-images. Consistency reflects whether a given pre-image abides by those restrictions. 
This section is organized as follows. Section~\ref{section:contributions:notions} introduces our key notions and our new problem formulation. Section~\ref{section:contributions:ILP} presents our technique and provides optimality guarantees. 

\subsection{Notions and Problem Statement} \label{section:contributions:notions}

We start by introducing the notion of a proximity graph.
Let $h$ be an encoder from $\+X$ to $\mathbb{R}^m$.  
 
\begin{definition}[Proximity graphs] \label{definition:lattice-graph}
    A \textit{proximity graph} $\+G_{\+D}^{h}$ for $\+D$ subject to $h$ is a directed graph that includes one node $(\ell,x)$, for each $\ell \in [n]$ and $x \in \_x_{\ell}$, and, optionally, a directed edge from node $(\ell,x)$ to node $(\ell',x')$ if $h(x')$ is close to $h(x)$,
    for $x,x' \in \+X$.
\end{definition}

The edges of the graph $\+G_{\+D}^{h}$ define proximity in the representation space. Notice that Definition~\ref{definition:lattice-graph} does not depend on either the encoder $h$ that will give us the latent representations, e.g., the encoder can be a pre-trained large vision and language model such as BLIP-2 \cite{Blip2}, or on the measure used to decide the distance in the representation space. 
We deliberately kept the vague term ``close'' in Definition~\ref{definition:lattice-graph} to support any distance measure a user may prefer.  
For example, an option is to define a distance threshold $\theta$ and add edges only between instances whose latent representations are less than $\theta$ apart.
A second option is to add a directed edge ${(\ell,x) \rightarrow (\ell',x')}$ only if $h(x')$ is in the top-$k$ neighborhood of $h(x)$, for $x,x' \in \+X$ -- the use of directed edges gives us greater flexibility to adopt such definitions. Of course, the ``better'' the encoder $h$ is, 
the more effective our algorithm will be in pruning the non-gold pre-images. 
To generalize to multiple encoders, we simply need to use different encoders to decide when two instances are close in the representation space instead of using a single $h$. 

The graph $\+G_{\+D}^{h}$  tells us when two instances of different samples are very close in the representation space.
When two instances are very close in the representation space, they should be of the same class, sharing the same gold labels. Due to the dependencies among different labels in the pre-images, some candidate pre-images may satisfy the restriction that the corresponding instances should share the same gold labels, while others may not.
The notion of \textit{consistency} formalizes the above intuition. 

\begin{definition}[Consistency] \label{definition:proof-consistency}
    For a proximity graph $\+G_{\+D}^{h}$, a pre-image $\sigma_{\ell,i}$ in $\+D$ is \textit{consistent} with an edge ${(\ell,x) \rightarrow (\ell',x')}$ in $\+G_{\+D}^{h}$
    if there exists a pre-image $\sigma_{\ell',i'}$ in $\+D$, such that ${\sigma_{\ell,i}(x) = \sigma_{\ell',i'}(x')}$ holds; otherwise, we say that $\sigma_{\ell,i}$ is \textit{inconsistent} with ${(\ell,x) \rightarrow (\ell',x')}$. 
    The pre-image $\sigma_{\ell,i}$ is \textit{globally consistent in $\+G_{\+D}^{h}$} if there does not exist an edge ${(\ell,x) \rightarrow (\ell',x')}$ in $\+G_{\+D}^{h}$ with which $\sigma_{\ell,i}$ is inconsistent. 
\end{definition}

We present an example of Definition~\ref{definition:proof-consistency}.

\begin{example} [Contd Example~\ref{example:NSL3}] 
\label{example:NSL4}
    Assume the proximity graph for the two training samples in our running example includes edge $e_1 := (1,x_1) \rightarrow (2,x'_1)$. 
    Since there does not exist a pre-image associated with the second training sample mapping $x'_1$ to 8, the pre-image ${\sigma_{1,3} = \{x_1 \mapsto 8, x_2 \mapsto 0 \}}$ is inconsistent with $e_1$.  
    In contrast, the pre-image ${\sigma_{1,1} = \{x_1 \mapsto 0, x_2 \mapsto 8 \}}$ is consistent with $e_1$, due to the existence of the pre-image ${\sigma_{2,1} = \{x'_1 \mapsto 0, x'_2 \mapsto 2\}}$.
    Generalizing this example, all the pre-images in the first training sample that map $x_1$ to a digit greater than 2 are inconsistent with $e_1$. The remaining pre-images are consistent with $e_1$. 
    Now, consider the edge $e'_1 := (2,x'_1) \rightarrow (1,x_1)$. 
    In the absence of other edges, all pre-images of the second sample are globally consistent.
\end{example}
Inconsistencies between pre-images and edges indicate violations of the restriction that the corresponding instances belong to the same class as we have seen in our running example. Hence, the corresponding pre-images need to be discarded. Definition~\ref{definition:lattice-graph-pruning} summarizes the process of discarding pre-images from a set of \nesy~samples based on such inconsistencies. 

\begin{definition}[Pruning] \label{definition:lattice-graph-pruning}  
    The \textit{pruning} $\pruning{\+G_{\+D}^{h}}$ of $\+D$
    subject to $\+G_{\+D}^{h}$ 
    is the set of \nesy~samples that results after 
    removing from each training sample in $\+D$ each pre-image that is
    inconsistent with an edge in $\+G_{\+D}^{h}$. The pruning is \textit{sound} if at least one pre-image is preserved for each sample.
\end{definition}

One might think that a strategy for discarding pre-images from $\+D$ would be the following: (1) Construct a proximity graph $\+G^{h}_{\+D}$ including as many edges as possible\footnote{Under the assumption that the instances $x, x'$ are in fact close under $h$ and the distance measures in use, where $x,x' \in \+X$.}; 
and (2) Remove each pre-image that is inconsistent with an edge in $\+G^{h}_{\+D}$.
However, such an approach \emph{does not} result in a sound pruning, as we demonstrate below:
\begin{example}[Contd Example~\ref{example:NSL4}] \label{example:NSL5}
    Consider also a third training sample $(\{x''_1,x''_2\}, \phi_3 := l''_1 + l''_2 = 16)$ and the edge ${e_2 := (1,x_1) \rightarrow (3,x''_1)}$. 
    In the pruning of the proximity graph that includes both $e_1$ (see Example~\ref{example:NSL4}) and $e_2$, the first training sample will be associated with zero pre-images. This is because $x_1$ cannot range simultaneously in the domains $\{0,1,2\}$ and $\{7,8,9\}$.
\end{example}

Cases such as those described in Example~\ref{example:NSL5} are met when the encoder maps instances of difference classes very close in the representation space. 
In other words, while adding as many edges as possible to $\+G^*_{\+D}$ does not affect the soundness of $\pruning{\+G^*_{\+D}}$, this property does not hold in the general case. 

To summarize the discussion so far, discarding pre-images from a set of \nesy~training samples reduces to finding a proximity graph whose edges reflect proximity in the representation space, according to Definition~\ref{definition:lattice-graph}. However, we need to be careful on how we choose this proximity graph:
too few edges may result in discarding very few pre-images; too many edges may result to prunings that are not sound, see Definition~\ref{definition:lattice-graph-pruning}.
The above gives rise to the following optimization problem. 

\begin{problem}\label{problem}
    For an encoder $h$, find the proximity graph $\+G^{h}_{\+D}$ that leads to the pruning of $\+D$ that (1) is sound, (2) includes all globally consistent pre-images, and (3) has the lowest total number of pre-images across all training samples.
\end{problem}

\textbf{Properties of our formulation.} According to Problem~\ref{problem}, the desired proximity graph should maximize the number of discarded pre-images. Soundness ensures that we still have at least one pre-image in each training sample and, hence, we can use those samples for training. This assumption comes from the fact that, by definition, each \nesy~training sample includes the gold pre-image. 
Due to the above, supervised samples are guaranteed to keep their single, gold pre-image. 
We require the pruning of $\+D$ to include all globally consistent pre-images as we have no evidence to discard them.  
The retained pre-images in $\pruning{\+G^{h}_{\+D}}$ adhere to the constraints since we only remove pre-images from the training samples and all pre-images adhere to the constraints by definition, Section~\ref{section:preliminaries}.
Finally, \ours{} may discard the gold pre-images. However, as our empirical analysis shows,  Table~\ref{tab:abl_sum3_dolphin} and Table~\ref{tab:abl_sum3_scallop} (appendix), \ours{} maintains the gold pre-images for more than 90\% of the cases on average. 

{
\begin{algorithm}[tb]
    \caption{\textsc{Training under} \ours} \label{algorithm:top-k-lattice-pruning}
	\begin{algorithmic}
		\STATE \textbf{Inputs:} Encoder $h$; Classifier $f$; \# training epochs $T$; {If $h$ is trainable then $f$ is formed by connecting the outputs of $h$ with a classification layer; otherwise, $h$ and $f$ are independent.} 
        \\ \nesy~dataset $\+D = \{(\_x_\ell, \{\sigma_{\ell,i}\}_{i=1}^{\omega_\ell} )\}_{\ell=1}^n$ \\ loss function $\mathsf{loss}$.
            \STATE \textbf{Outputs:} Trained $f$. 
            \STATE $f_0 := f$
            \FOR{\textbf{each} $t \in [T]$}
                \FOR{\textbf{each} mini-batch $\mathbf{b}$ of $\+D$}
                    \STATE \textbf{find} the proximity graph $\+G_{\mathbf{b}}^{h}$ for $\mathbf{b}$ maximizing \eqref{eq:lattice-graph-ilp}.
                    \FOR{\textbf{each} $\ell \in [n]$}
                        \STATE $\Omega_\ell := \emptyset$ 
                        \FOR{\textbf{each} $i \in [\omega_\ell]$}
                            \STATE \textbf{add} $\sigma_{\ell,i}$ to $\Omega_\ell$ if ${ I'_{\ell,i} } = 0$ in the optimal solution to \eqref{eq:lattice-graph-ilp}.
                        \ENDFOR
                        \STATE $l := \mathsf{loss}(f_{t}(\_x_\ell),\Omega_\ell)$
                        \STATE $f_{t+1} := \mathsf{backpropagate}(f_t,\nabla l)$
                    \ENDFOR
                \ENDFOR
            \ENDFOR
            \STATE \textbf{return} $f_T$
	\end{algorithmic}
    
\end{algorithm}
}

\subsection{A Linear Programming Formulation} \label{section:contributions:ILP}

Recall that naive greedy pruning compromises soundness (Example~\ref{example:NSL5}), hence, it's not an option to solve Problem~\ref{problem}.
To formalize Problem~\ref{problem} as an ILP, we need to define the binary variables. 
First, we add a binary variable $E_{\ell,\ell'}^{x,x'}$ for each ${\ell, \ell' \in [n]}$, ${x \in \_x_\ell}$, and ${x' \in \_x'_\ell}$, if $h(x')$ is close to $h(x)$ -- $h$ and ``closeness'' is an implementation choice as discussed in Section~\ref{section:contributions:notions}. 
The variable $E_{\ell,\ell',x, x'}$ is one if the resulting proximity graph includes the corresponding edge and zero otherwise.  
Second, we add a binary variable $I_{\ell, i}$ that corresponds to $\sigma_{\ell,i}$, that is the $i$-th pre-image of the $\ell$-th training sample, for ${\ell \in [n]}$ and ${i \in [\omega_{\ell}]}$. The variable $I_{\ell, i}$ is one if $\sigma_{\ell,i}$ is in the pruning of $\+D$ subject to the resulting proximity graph;  otherwise it is zero. 
Finally, we add a binary variable $I'_{\ell, i}$ for each ${\ell \in [n]}$ and ${i \in [\omega_{\ell}]}$ that is the complement of $I_{\ell, i}$, i.e., it is one when $I_{\ell, i}$ is zero and vice versa. 
We now discuss the constraints of the linear program. 

The first constraint is $I_{\ell,i}  +  I'_{\ell,i} = 1$ and states that the two variables are mutually exclusive. 
The second constraint is $\sum_{i=1}^{[\omega_\ell]}  I_{\ell,i}  \geq 1$, for each $\ell \in [n]$, and states that each training sample must include at least one pre-image. 
The third constraint is $I_{\ell,i} = 1$, for each $\ell \in [n]$ and $i \in [\omega_\ell]$, if $\sigma_{\ell,i}$ is globally consistent in any proximity graph that can be computed for the given training samples $\+D$
subject to $h$ and the distance measures selected, see Definition~\ref{definition:proof-consistency}. This constraint ensures that those pre-images will not be discarded in the pruning. 
The fourth constraint is 
${ E_{\ell,\ell'}^{x,x'}  +  I_{\ell,i}  = 1}$ (comes from ${1 -  E_{\ell,\ell'}^{x,x'}  + 1 -  I_{\ell,i}  = 1}$) and expresses that $\sigma_{\ell,i}$ is inconsistent with the edge ${(\ell,x) \rightarrow (\ell',x')}$, see Definition~\ref{definition:proof-consistency}. 
The remaining constraints define the domain.
We aim to maximize the number of discarded pre-images. 

\vspace{-0.1in}
\begin{align}
\textbf{Objective } & \max \sum_{\ell \in [n], i \in  [\omega_\ell]}  I'_{\ell,i}  & \textbf{s.t. } \label{eq:lattice-graph-ilp} 
\end{align}
\begin{small}
\begin{flalign}
    \begin{array}{rll}
         I_{\ell,i}  +  I'_{\ell,i}  & = 1, & \forall \ell \in [n], \forall i \in  [\omega_\ell] \\ 
         \sum \limits_{i=1}^{[\omega_\ell]}  I_{\ell,i}  & \geq 1, & \forall \ell \in [n] \\
         I_{\ell,i}    & = 1, &\forall \ell \in [n], \forall i \in  [\omega_\ell], \text{s.t.}  \\
         & & \text{$\sigma_{\ell,i}$ is always globally consistent.}\\
         E_{\ell,\ell'}^{x,x'} +  I_{\ell,i}  & = 1, & \text{$\forall \ell \in [n], \forall \ell' \in [n], \forall x \in \_x_\ell, \forall x' \in \_x_{\ell'}$, s.t.} \\
         & & \text{$\sigma_{\ell,i}$ is inconsistent with ${(\ell,x) \rightarrow (\ell',x')}$.} \\
         E_{\ell,\ell'}^{x,x'}  & \in \{0,1\}, &\forall \ell \in [n], \forall \ell' \in [n], \forall x \in \_x_\ell, \forall x' \in \_x_{\ell'}, \text{s.t.}  \\
          & & \text{$h(x')$ is close to $h(x)$.} \\
         I_{\ell,i}   & \in \{0,1\}, &\forall \ell \in [n], \forall i \in  [\omega_\ell] \\
         I'_{\ell,i}   & \in \{0,1\}, &\forall \ell \in [n], \forall i \in  [\omega_\ell]
    \end{array} \nonumber
\end{flalign}  
\end{small}

We formalize correctness below.
\begin{restatable}{proposition}{optimality}[Optimality] \label{proposition:ilp-optimality}
    The solution to \eqref{eq:lattice-graph-ilp} is the optimal solution of Problem~\ref{problem}.
\end{restatable}
Algorithm~\ref{algorithm:top-k-lattice-pruning} summarizes our technique for pruning pre-images from a set \nesy~training samples and training a classifier $f$. 
We can train the encoder $h$ simultaneously with $f$ or use it only to decide the proximity in the representation space. In the former case, we assume that $f$ is composed by linking the outputs of $h$ with a classification layer. The loss function is a standard \nesy{} loss function, e.g. semantic loss \cite{pmlr-v80-xu18h}, its top-$k$ approximation \cite{scallop}, or a fuzzy loss \cite{Dolphin,LTNs}, that computes the level of ``adherence" 
to the background knowledge (represented by the retained pre-images $\Omega_\ell$) of the output scores of the classifier $f$ given $\_x_\ell$. 
The algorithm works by solving \eqref{eq:lattice-graph-ilp} for each mini-batch of $\+D$, using the retained pre-images in $\Omega_{\ell}$ for loss computation, and backpropagating through $f$ -- if $h$ is connected to the output layer of $f$, then $h$ is also trained. 

The complexity of the state-of-the-art \nesy{} loss, semantic loss \cite{pmlr-v80-xu18h} (also implemented by \scallop), is \#P-complete \cite{CHAVIRA2008772}. ILP is an NP-hard problem. 
To reduce the computational overhead of our technique, we could adopt multiple heuristics \citep{ilp-heuristics}.
In practice, our formulation runs efficiently even without approximations, see Table~\ref{tab:abl_sum3_dolphin}.

\section{Experiments}\label{section:experiments}

\begin{table*}[tb]
    \begin{center}
    \caption{Classification accuracy for SUM-$M$.}
    \label{table:msum}
    \resizebox{\linewidth}{!}{
    \begin{threeparttable}
    \begin{tabular}{@{}l | c | c | c | c | c | c | c | c@{}}
    \toprule
    \multirow{2}{*}{\textbf{Algorithms}}                    &
    \multicolumn{2}{c|}{$n$=100, MNIST}             & 
    \multicolumn{2}{c|}{$n$=full, MNIST}             &
    \multicolumn{2}{c|}{$n$=10K, CIFAR-10}             & 
    \multicolumn{2}{c}{$n$=full, CIFAR-10}            \\
    & 
    $M=3$ & $M=4$ & 
    $M=3$ & $M=4$ &
    $M=3$ & $M=4$ & 
    $M=3$ & $M=4$ \\ 
    \midrule
    \scallop                & 
    43.47 ± 14.02 & 24.33 ± 7.08 &
    99.07 ± 0.02 & 99.07 ± 0.07 &
    88.70 ± 0.37 & 88.83 ± 0.46 &
    91.74 ± 0.27 & 90.66 ± 0.52\\
    $~$ + \oursfrozen{}       & 
    \textbf{60.06 ± 17.52} & \textbf{33.90 ± 8.14} &
    99.12 ± 0.05 & {99.21 ± 0.01} & 
    87.35 ± 0.29 & 88.55 ± 0.52 &
    90.63 ± 0.30 & 90.45 ± 0.21\\
    $~$ + \ourstrainable{}        & 
    60.03 ± 21.72 & 29.70 ± 10.39 &
    {99.22 ± 0.06} & {99.21 ± 0.08} &
    87.87 ± 0.66 & {88.91 ± 0.72} &
    91.37 ± 0.04 & 90.34 ± 0.03\\
    \midrule
    \dolphin             &  
    43.29 ± 14.15 & 24.63 ± 7.13 &
    99.09 ± 0.06 & 99.18 ± 0.02 &
    88.54 ± 0.52 & 88.96 ± 0.67 &
    91.77 ± 0.12 & 90.75 ± 0.61 \\
    $~$ + \oursfrozen{}    & 
    60.14 ± 17.68 & \textbf{33.96 ± 8.16} &
    99.10 ± 0.07 & 99.25 ± 0.01 &
    87.71 ± 0.35 & 88.66 ± 0.18 &
    90.71 ± 0.19 & 90.39 ± 0.29 \\
    $~$ + \ourstrainable{}        & 
    \textbf{60.49 ± 21.86} & 30.30 ± 12.57 &
    {99.27 ± 0.04} & {99.24 ± 0.03} &
    87.83 ± 0.40 & 88.77 ± 0.15 &
    91.00 ± 0.41 & 90.45 ± 0.20 \\
    \midrule
    \upenn             &  
    19.16 ± 14.19 & 16.38 ± 8.16 &
    98.95 ± 0.17 & 98.82 ± 0.19 &
    64.74 ± 2.53 & 25.72 ± 10.21 &
    63.98 ± 1.87 & 34.78 ± 1.85\\
    $~$ + \oursfrozen{}    & 
    \textbf{24.98 ± 5.54} & 16.45 ± 2.25 &
    {98.96 ± 0.21} & {98.97 ± 0.36} & 
    \textbf{67.31 ± 1.82} & 40.40 ± 0.85 &
    66.80 ± 1.16 & 39.42 ± 5.50\\
    $~$ + \ourstrainable{}        & 
    21.07 ± 1.14 & {16.57 ± 10.60} &
    98.91 ± 0.16 & 98.65 ± 0.22 & 
    65.90 ± 2.67 & \textbf{41.63 ± 2.59} &
    \textbf{67.24 ± 2.87} & \textbf{42.67 ± 2.82}\\
    \midrule
    {\ablsim} &
    16.05 ± 0.60 & 11.02 ± 0.16 & 98.85 ± 0.21 & 98.98 ± 0.05 &
    86.19 ± 0.29 & 35.64 ± 0.85 & 89.64 ± 0.54 & 87.57 ± 0.29
    \\
    \bottomrule 
    \end{tabular}
    \end{threeparttable}
    }
    \end{center}
    \vspace{-0.1in}
\end{table*}

\begin{table*}[tb]
    \begin{center}
    \caption{Classification accuracy for SUM-$M$ and MAX-$M$ for the full dataset.}
    \label{table:msum-mmax-full}
    \resizebox{\linewidth}{!}{
    \begin{threeparttable}
    \begin{tabular}{@{}l | c | c | c | c | c | c | c | c@{}}
    \toprule
    \multirow{2}{*}{\textbf{Algorithms}}                    &
    \multicolumn{2}{c|}{SUM-$M$, MNIST}             & 
    \multicolumn{2}{c|}{MAX-$M$, MNIST}             &
    \multicolumn{2}{c|}{SUM-$M$, CIFAR-10}             & 
    \multicolumn{2}{c}{MAX-$M$, CIFAR-10}            \\
    & 
    $M=30$ & $M=40$ & 
    $M=30$ & $M=40$ & 
    $M=15$ & $M=30$ & 
    $M=15$ & $M=30$ \\ 
    \midrule
    \dolphin             &  
    73.39 ± 23.05 & 75.37 ± 15.86 &
    17.37 ± 5.22 & 11.47 ± 1.42 &
    18.94 ± 9.54 & 12.61 ± 2.41 &
    28.32 ± 2.11 & 15.59 ± 3.08  \\
    $~$ +  \oursfrozen{}    & 
    \textbf{85.52 ± 2.95} & 80.73 ± 4.71 &
    19.66 ± 2.67 & \textbf{25.29 ± 1.52} &
    21.81 ± 3.45 & 16.17 ± 4.54 &
    22.15 ± 3.48 & 15.42 ± 1.39  \\
    $~$ + \ourstrainable{}        & 
    81.25 ± 9.82 & \textbf{81.70 ± 6.00} &
    \textbf{25.46 ± 5.87} & 20.74 ± 2.20 &
    \textbf{30.16 ± 2.47} & \textbf{17.40 ± 2.80} &
    25.71 ± 2.39 & \textbf{17.27 ± 1.03}  \\
    \midrule
    {\ablsim} & TO & TO & TO & TO & TO & TO & TO & TO \\
    \bottomrule 
    \end{tabular}
    \end{threeparttable}
    }
    \end{center}
    \vspace{-0.1in}
\end{table*}

\begin{table}
    \centering
    \caption{Classification accuracy for HWF-$M$ for the full dataset.}
    \label{table:hwf}
    \resizebox{\columnwidth}{!}{
    \begin{threeparttable}
    \begin{tabular}{l | c | c}
    \toprule
    \textbf{Algorithms}   &
    {$M$=7, full dataset}               & 
    {$M$=15, full dataset}             \\
    \midrule
    \dolphin             &  
    22.74 ± 2.48 & 17.01 ± 2.86 \\
    $~~~$ +\oursfrozen{}    & 
    \textbf{33.20 ± 1.57} & \textbf{23.43 ± 6.11} \\
    $~~~$ + \ourstrainable{}         &
    32.29 ± 2.14 & 22.75 ± 1.61 \\
    \bottomrule 
    \end{tabular}
    \end{threeparttable}
    }
    \centering
    \vspace{0.1in}
    \caption{Classification accuracy for MUGEN.}
    \label{table:mugen}
      \resizebox{\columnwidth}{!}{
      \scriptsize
    \begin{tabular}{l | c | c | c | c }
    \toprule
    \multirow{2}{*}{\textbf{Algorithms}}             &
    \multicolumn{2}{c|}{VTR}             & 
    \multicolumn{2}{c}{TVR}             \\
    & 
    $n$=250 & $n$=500  &
    $n$=250 & $n$=500 \\ 
    \midrule
    \scallop                & 
    25.2 & 26.39    & 
    27.5 & 31.10  \\
    $~$ + \ourstrainable{}       & 
    \textbf{27.1} & \textbf{74.9}   & 
    \textbf{32.00} & \textbf{77.7}  \\
    \midrule
    \dolphin             &  
    33.8 & 33.9   & 
    34.8 & 34.7 \\
    $~$ + \ourstrainable{}       & 
    \textbf{83.7} & \textbf{86.6}   & 
    \textbf{85.0} & \textbf{86.9}  \\
    \bottomrule 
    \end{tabular}
      }
      \vspace{-0.15in}
\end{table}

\textbf{Benchmarks.} We consider a wide range of benchmarks. The first two, \textbf{SUM-$M$} and \textbf{MAX-$M$}, are two classic benchmarks in the literature \cite{deepproblog-journal}. SUM-$M$ has been used in our running example, while MAX-$M$ considers the maximum instead of the sum of the gold labels. 
In the above scenarios, the number of pre-images may be particularly large, making the supervision rather weak, e.g., in the MAX-4 scenario, there are $4 \times 9^3$ candidate combinations of labels when the weak label is $9$.
To assess the effectiveness of our technique under more complex representations, we also consider a variant of those benchmarks, where we associate each digit in $\{0,\dots,9\}$ with a CIFAR-10 class. 

The next benchmark is \textbf{HWF-$M$} \cite{hwf}. Each training sample consists of (1) a sequence ${(x_1,\dots,x_K)}$ of digits and mathematical operators in ${\{+,-,*\}}$, corresponding to a mathematical expression of length $M$ and (2) the result of the corresponding mathematical expression. The goal is to train a classifier to recognize digits and operators. This results in a task which is more difficult than the HWF benchmark discussed in \cite{Dolphin}, where the training samples are equations of length \emph{up to} $M$.

Our third benchmark is \textbf{VQAR} \cite{scallop}. 
VQAR extends GQA \cite{GQA} with queries that require multi-hop reasoning using knowledge from CRIC \cite{CRIC}. 
The benchmark includes the classifiers $name$ and $rel$ that return the type of an object within a given bounding box and the relationship between the objects within a pair of bounding boxes.
The objective is to train the above classifiers using samples of the form $(\_o, \phi)$, 
where $\_o$ are bounding boxes and $\phi$ is a sentence the bounding boxes abide by.
The benchmark includes 500 object types and 229 different relations. 
We restrict to the top-$k$ most frequent object types and relations for $k=\{50,100\}$.

Our last benchmark is \textbf{MUGEN} \cite{MUGEN}. %
Each training sample consists of a sequence of $N$ video frames and a sequence of $K$ actions that describe what the character does. 
The objective is to train a classifier to recognize the action in each frame. 
In general, $K <= N$, i.e., the same action may take place in more than one video frame. However, we do not exactly know which action takes place in each frame. We use two tasks to assess the performance of the classifier: video-to-text retrieval (VTR) and text-to-video retrieval (TVR). In VTR, given a video and $M$ sequences of actions, the classifier must choose the sequence most aligned with the video. 
In TVR, given a sequence of actions and $M$ videos, the classifier must choose the video most aligned with the action sequence. We measure accuracy by counting the number of times the classifier chose the ground-truth sequence of actions and videos.

\textbf{Engines and Variants.} 
We consider the state-of-the-art engines \scallop~\cite{scallop}, \dolphin~\cite{Dolphin}, 
\upenn~\cite{ised}, and \ablsim{} \cite{ABLSim}. 
We do not consider \cite{deepproblog-journal,ABL,neurasp} due to scalability issues \cite{NeurIPS2023}. 
For the same reason, we do not apply semantic loss \cite{pmlr-v80-xu18h}, but use its approximation offered in \scallop.
For \dolphin, we use fuzzy semantics via the DAMP provenance, as it was empirically shown to lead to better performance \cite{Dolphin}.
We apply \ours{} with and without training the encoder $h$, see Algorithm~\ref{algorithm:top-k-lattice-pruning}. 
In the former case (denoted by \ourstrainable), $h$ is randomly initialized; in the latter (denoted by \oursfrozen), $h$ is pre-trained and frozen. 
\oursfrozen{} is not applicable in MUGEN as the benchmark does not provide access to the gold pre-images. 
\ablsim{} is tightly dependent on \abl{}, hence it's integration with \scallop, \dolphin, or \upenn{} is not straightforward.   
More details are in the appendix. 
 
We assess the performance of \ours{} using the classification accuracy of the underlying classifiers. In SUM-$M$, MAX-$M$, and HWF-$M$, the results are obtained over three runs. In VQAR and MUGEN, each experiment was run once, following \cite{scallop-journal}. 
The results are shown in Tables~\ref{table:msum}-\ref{table:mugen}. 
In Tables~\ref{table:msum}, \ref{table:vqar}, and \ref{table:mugen}, we used all pre-images to train under \scallop{} and \dolphin{}. 
\upenn{} and \ablsim{} implement sampling-based learning only, so we trained subject to the sampled pre-images, i.e., we apply \ours{} on the sampled pre-images. 
In Tables~\ref{table:msum-mmax-full} and \ref{table:hwf},
we considered only \dolphin{} -- \scallop{} does not scale for those scenarios \cite{Dolphin} -- and trained it by randomly sampling the pre-images to avoid the search space explosion.
Finally, Tables~\ref{tab:abl_sum3_dolphin} and \ref{tab:abl_sum3_scallop} (appendix) provide additional statistics on the SUM-3 scenario from Table~\ref{table:msum}, namely, (i) the mean percentage of retained pre-images after pruning per batch, (ii) the mean percentage of retained pre-images including the gold ones after pruning per batch, (iii) the mean total time to form the ILP in \eqref{eq:lattice-graph-ilp} and solve it (column ``Pruning time") per epoch, and (iv) the mean runtime overhead introduced over the baseline model (column ``\% of Training Time") per epoch. In addition, they provide ablations for different encoders and batch sizes. All experiments ran for 50 epochs.
We highlight the cases where the difference with the baseline accuracy is more than $1\%$. 
We provide further information in the appendix.

\begin{table*}[]
    \centering
    \caption{{Ablations for MNIST SUM-3 and $n=100$ for different pretrained and frozen encoders and batch sizes.}}
    \resizebox{\linewidth}{!}{
    \begin{tabular}{c|l|c|c|c|c|c}
    \toprule
         \makecell{\bf Batch size} & \textbf{Algorithms} & \makecell{\bf Accuracy} & \makecell{\bf \% Retained pre-images} & \makecell{\bf \% Retained gold pre-images} & \makecell{\bf Pruning time (s)} & \makecell{\bf \% of Training Time} \\
         \midrule
        \multirow{4}{*}{32} & \dolphin &
        52.88 ± 18.28 & NA & NA & NA & NA\\
         & $~$ + \oursfrozen(\resnet-18) &
         45.45 ± 6.45 & 79.97 ± 2.07 & 92.47 ± 1.13 & 0.98 ± 0.02 & 18.41 ± 0.07\\
         & $~$ + \oursfrozen(\resnet-152) &
         \textbf{67.82 ± 12.09} & 79.20 ± 2.69 & 89.33 ± 2.26 & 1.17 ± 0.02 & 17.47 ± 0.18\\
         & $~$ + \ourstrainable &
         44.53 ± 10.36 & 79.02 ± 2.31 & 94.65 ± 1.76 & 0.52 ± 0.00 & 19.15 ± 0.17\\
         \midrule
         \multirow{4}{*}{64} & \dolphin &
         43.29 ± 14.15 & NA & NA & NA & NA\\
         & $~$ + \oursfrozen(\resnet-18) &
         60.14 ± 17.68 & 80.91 ± 1.90 & 94.35 ± 1.74 & 0.88 ± 0.02 & 19.25 ± 0.34\\
         & $~$ + \oursfrozen(\resnet-152) &
         49.47 ± 8.93 & 78.84 ± 2.87 & 90.35 ± 2.31 & 1.06 ± 0.01 & 18.43 ± 0.20\\
         & $~$ + \ourstrainable &
         \textbf{60.49 ± 21.86} & 79.33 ± 1.67 & 95.65 ± 1.22 & 0.50 ± 0.01 & 18.78 ± 0.43\\
         \midrule
         \multirow{4}{*}{128} & \dolphin &
         35.73 ± 15.54 & NA & NA & NA & NA\\
         & $~$ + \oursfrozen(\resnet-18) &
         38.68 ± 6.96 & 82.05 ± 1.46 & 95.51 ± 2.49 & 0.78 ± 0.03 & 18.53 ± 0.13\\
         & $~$ + \oursfrozen(\resnet-152) &
         41.65 ± 12.99 & 79.29 ± 3.09 & 91.67 ± 2.05 & 0.91 ± 0.01 & 17.84 ± 0.14 \\
         & $~$ + \ourstrainable &
         \textbf{53.50 ± 17.13} & 77.92 ± 0.87 & 96.12 ± 0.41 & 0.51 ± 0.01 & 19.18 ± 0.30\\
         \bottomrule
    \end{tabular}
    }
    \label{tab:abl_sum3_dolphin}
    \vspace{-0.1in}
\end{table*}

\begin{table} [t]
    \caption{name (N)/relation (R) classification accuracies for VQAR. i3 and i5 denote the types of sentences in $\+D$.}
    \label{table:vqar}
    \resizebox{\columnwidth}{!}{
    \begin{threeparttable}
    \scriptsize
    \begin{tabular}{l | c | c | c}
        \toprule
        \textbf{Algorithms}                 &
        {top-50 N, top-50 R}            & 
        {top-100 N, top-50 R}             &
        {top-100 N, top-50 R} \\
                 &
        {$n=1000$ }            & 
        {$n=1000$ }             &
        {$n=5000$ } \\
        \midrule
        \scallop                &  
        46.58/19.93  & 35.6/{12.62} & 37.98/13.94  \\
        $~$ + \oursfrozen{}       &  
        \textbf{48.08}/\textbf{22.41}  & {36.17}/12.55 & \textbf{39.70}/{14.32}  \\
        \bottomrule 
        \end{tabular}
    \end{threeparttable}
    }
    \vspace{-0.15in}
\end{table}

\textbf{Discussion.} 
\ours~can bring substantial increases in accuracy for all baseline models, showing its robustness to the baseline model.
For example, in SUM-3, Table~\ref{table:msum}, the mean classification accuracy of the baseline \scallop~model increases from 43.47\% to 60.06\% when $n=100$, and MNIST digits are used. For the baseline \dolphin{} model, it increases from 43.29\% to 60.49\%, while for the baseline $\upenn$ model it increases from 19.16\% to 24.98\%. 
For HWF-7 in Table~\ref{table:hwf}, the accuracy of the baseline \dolphin~model increases from 22.74\% to 33.20\%. In Table~\ref{table:mugen}, VTR increases for the baseline \scallop~model from 26.39\% to 74.9\% when $n=500$; for \dolphin, this increase is from 33.9\% to 86.6\%. 
The baseline accuracy of \scallop{} that is reported in Table~\ref{table:vqar} is less than the one reported in \cite{scallop}. This is because (1) \cite{scallop} does not report classification accuracy, but accuracy in the overall output and (2) in \cite{scallop}, the samples in $\+D$ are easier: in our analysis, the training samples are associated with more than $100$ pre-images on average. 
Unlike \ours, \ablsim{} can lead to poorer accuracy than the baseline models, especially in the small data regime, $n=100$ with MNIST and $n=10$K with CIFAR-10 in Table~\ref{table:msum}, and 
times out during pre-image computation for the experiments in Tables~\ref{table:msum-mmax-full}, taking more than an hour to process the first batch.

Our empirical analysis shows that \ours{}: (1) brings significant improvements both under pre-trained and randomly initialized \& trainable encoders; and (2) is robust to the underlying encoder. Regarding (1), Tables~\ref{table:msum}, \ref{table:msum-mmax-full} and \ref{table:hwf} show that \oursfrozen{} and \ourstrainable{} are on par in many cases. For example, in Table~\ref{table:hwf}, 
for $M=7$, the baseline accuracy improves from 22.74\% to 33.20\% under \oursfrozen{} and to 32.29\% under \ourstrainable. 
Regarding (2), consider Tables~\ref{tab:abl_sum3_dolphin} and \ref{tab:abl_sum3_scallop} (appendix): out of the six cases in total, only in one case (when the batch size is 32 in Table~\ref{tab:abl_sum3_dolphin}) the change of the pre-trained encoder leads to worse performance. 

Lastly, \ours{} (1) is quite effective in reducing the number of pre-images and retaining the gold ones and (2) incurs small runtime overhead.  
Regarding the first point, \oursfrozen{} and \ourstrainable{} remove, on average, from $\sim 18\%$ to up to $\sim 23\%$ of the pre-images, while, on average, from $\sim 89\%$ to up to $\sim 96\%$ of the gold pre-images are maintained. The percentage of gold pre-images that are maintained under \ourstrainable{} manifests its robustness to simultaneously learn the encoder and use it to remove non-useful label combinations. 
Regarding the second point, we can see from Tables~\ref{tab:abl_sum3_dolphin} and \ref{tab:abl_sum3_scallop} (appendix) that \ours{} increases by up to $\sim$ 19\% the runtime of \dolphin{} and by up to $\sim$ 6\% the runtime of \scallop. 
In the former case, the relative runtime  overhead is higher, as \dolphin{} has a lower runtime.  
The absolute runtime overhead is roughly the same in both cases (around 1 second per epoch on average).

\section{Related Work}\label{section:related}

\textbf{\nesy~Learning.}
Research on \nesy~mainly focused on developing efficient \nesy~losses \cite{pmlr-v80-xu18h,LTNs} and sampling-based training techniques \cite{iclr2023,ised,ABL}. 
The authors in \cite{reasoning-shortcuts} proposed different strategies to anticipate the lack of gold labels during training. However, most of these strategies made additional assumptions during training. The only strategy in \cite{reasoning-shortcuts} that exploited the representation space was the one employing an autoencoder-based loss (Section 5.3). Unlike \ours, this strategy requires modifying the classifier’s architecture and cannot operate in a training-free manner. 
The work in \cite{bears} trains multiple classifiers to improve label disambiguation during training. The latter research is orthogonal to ours.
Similarly, our work is orthogonal to the research in \cite{Semantic-Probabilistic-Layers,li-srikumar-2019-augmenting,pmlr-v180-ahmed22a}, as none of them aims to reduce the number of pre-images used for training.

\ablsim{} \cite{ABLSim} is, to our knowledge, the closest technique to \ours{}.
However, it has multiple technical differences. 
First, \ablsim's consistency measure (Eq. 12 in \cite{ABLSim}) is different from our notion of consistency in Definition~\ref{definition:proof-consistency}. Second, our problem formulation, see Problem~\ref{problem}, and the solution to our problem, see \eqref{eq:lattice-graph-ilp}, are very different from \ablsim's objective (Eq. 12 in \cite{ABLSim}). 
In practice, \ablsim{} can perform worse than the baseline engines, Table~\ref{table:msum}. 

The \nesy{} literature has started investigating techniques that exploit the representation space for different tasks. For example, \vael{} is a \nesy{} generative model that relies on variational autoencoders for image generation \cite{vael}. \gedi{} blends self-supervised cluster-based learning with energy models to define a loss function to improve the quality of latent representations \cite{gedi}.
The above techniques are orthogonal to \ours{} as they do not focus on pruning pre-images. 
However, they could both be integrated with \ours{}, e.g.,
\ours{} can be integrated with \vael{} similarly to \scallop{} and \dolphin;
\gedi{} could be the training loss after pruning the pre-images under \ours.

\textbf{Partial Label Learning.}
Most relevant to ours is the work of He et al. \citeyear{he2024candidate}. However, their formulation (1) cannot be extended to support \nesy~and (2) neither discards the maximum number of pre-images across all training samples, as Proposition~\ref{proposition:ilp-optimality} does -- their proposed technique greedily eliminates the candidate labels from $Y$ that do not occur frequently in the top-$k$ neighbors of $x$ in the representation space. As we state in Section~\ref{section:contributions}, greedy pruning may compromise soundness (Example~\ref{example:NSL5}).

\section{Conclusions}\label{section:conclusions}

We introduced a technique to reduce the space of candidate label combinations in \nesy~by exploiting the proximity of instances in the representation space that is supported by a new problem formulation and an optimal solution. 
An option would be to use pretrained LLMs to infer the gold labels directly \cite{stein2025roadgeneralizableneurosymboliclearning}. Beyond being cost- and resource-demanding, this approach can be seen as a special case of our approach that retains only one pre-image: the one that best aligns with the LLM's predictions. 
Our formulation is a non-trivial extension to this setting, where we do not keep a single pre-image but multiple ones that abide by the background theory. Future research includes extending the theoretical analysis in \cite{he2024candidate} to our setting. 

\section*{Impact Statement}

Our work aims to advance the field
of machine learning. There are many potential societal 
consequences of it, none of which we feel must be
specifically highlighted here.

\bibliography{references}
\bibliographystyle{icml2026}

\newpage
\appendix
\onecolumn
\section{Notation} \label{appendix:notation}

\begin{table*}[h]
    \centering
    \caption{The notation in the preliminaries and the proposed algorithm.}\label{table:notation}
    \begin{tabular}{|l|p{9cm}|}
        \multicolumn{2}{c}{Supervised learning notation} \\
        \hline
        $[n] := \{1,\dots,n\}$ & Set notation \\
        $\mathcal{X}$, $\mathcal{Y}$  & Input instance space and label space  \\
        $x,y$ & Elements from $\mathcal{X}$ and $\mathcal{Y}$ \\
        $\Delta_c$ & Space of probability distributions over $\+Y$ \\
        ${f: \+X \mapsto \Delta_c}$ & Scoring function \\
        ${f^j(x)}$  & Score of $f$ upon $x$ for class $j \in \+Y$\\
        \hline
        \multicolumn{2}{c}{NSL notation} \\ \hline
        $\+D$ & Set of \nesy~training samples \\
        $n$ & Number of samples in $\+D$ \\ 
        $\ell,\ell'$ & Indices over $[n]$ \\ 
        $\_x_\ell$ & The vector of instances in the $l$-th \nesy~sample in $\+D$ \\
        $\sigma_{\ell,i}$ & The $i$-th pre-image of the $\ell$-th \nesy~sample in $\+D$ \\
        $\omega_\ell$ & Number of pre-images in the $\ell$-th \nesy~sample in $\+D$ \\
        \hline 
        \multicolumn{2}{c}{Notation in Section~\ref{section:contributions}} \\ \hline  
        $h$ & Encoder from $\+X$ to $\mathbb{R}^m$ \\
        $\+G_{\+D}^{h}$ & Proximity graph for $\+D$ subject to $h$ \\
        $E_{\ell,\ell'}^{x,x'}$ & Binary variable becoming 1 if ${(\ell,x) \rightarrow (\ell',x')}$ is in $\+G_{\+D}^{h}$ \\
        $I_{\ell,i}$, $I'_{\ell,i}$ & Binary variables corresponding to pre-image $\sigma_{\ell,i}$ in $\+D$ 
        \\ \hline
    \end{tabular}
\end{table*}

\section{Details on Section~\ref{section:contributions}} \label{appendix:proofs}

\optimality*
\begin{proof}
The proof proceeds by construction. Recall that the variable $I'_{\ell,i}$ corresponds to the $i$-th pre-image of the $\ell$-th training sample $\sigma_{\ell,i}$, for ${\ell \in [n]}$ and ${i \in [\omega_{\ell}]}$ and becomes one if $\sigma_{\ell,i}$ is \emph{not} in the pruning of $\+D$ subject to the resulting proximity graph, see Section~\ref{section:contributions:ILP}. Due to the above, the optimization objective ${ \max \sum \limits_{\ell \in [n], i \in  [\omega_\ell]}  I'_{\ell,i}  }$ in \eqref{eq:lattice-graph-ilp} aligns with the objective (3) in Problem~ \ref{problem}.
Now, let us move to objective (2) from Problem~ \ref{problem}, that is the pruning of $\+D$ subject to the resulting proximity includes all globally consistent pre-images. This objective is satisfied due to the third constraint in \eqref{eq:lattice-graph-ilp}.

Finally, let us move to objective (1) from Problem~\ref{problem}, that is the pruning of $\+D$ subject to the resulting proximity is sound, that is, each at least one pre-image is preserved for each sample, see Definition~\ref{definition:lattice-graph-pruning}.
From Section~\ref{section:contributions:ILP}, we know that the variable $E_{\ell,\ell'}^{x,x'}$, for each ${\ell, \ell' \in [n]}$, ${x \in \_x_\ell}$, and ${x' \in \_x'_\ell}$, (a) denotes that $h(x')$ is close to $h(x)$ and (b) becomes one if the resulting proximity graph includes the corresponding edge and zero otherwise.
The fourth constraint in \eqref{eq:lattice-graph-ilp}, that is, $1 -  E_{\ell,\ell'}^{x,x'}  + 1 -  I_{\ell,i}  = 1$, enforces that each pre-image $\sigma_{\ell,i}$ that is inconsistent with the edge ${(\ell,x) \rightarrow (\ell',x')}$ (and this edge is included in the resulting proximity graph) will not be included in the pruning of $\+D$ subject to the resulting graph. Notice that whenever the variable $E_{\ell,\ell'}^{x,x'}$ becomes one, the variable $I_{\ell,i}$ becomes zero. The above, along with the facts that (i) each training sample in the resulting pruning of $\+D$ is associated with at least one pre-image -- enforced by the constraint $\sum \limits_{i=1}^{[\omega_\ell]}  I_{\ell,i}  \geq 1$, for each $\ell \in [n]$ in \eqref{eq:lattice-graph-ilp} -- and (ii) each pre-image is either included in the resulting pruning or not -- enforced by the constraint $I_{\ell,i}  +  I'_{\ell,i}  = 1$, for each $\ell \in [n]$ and each $i \in  [\omega_\ell]$ in \eqref{eq:lattice-graph-ilp} --  ensure that the LP in \eqref{eq:lattice-graph-ilp} satisfies objective (1) from Problem~\ref{problem}, completing the proof of Proposition~\ref{proposition:ilp-optimality}. 
\end{proof}

\section{Further details on the experiments} \label{appendix:experiments}

\textbf{Benchmarks.} In \textbf{SUM}-$M$ and \textbf{MAX}-$M$ training samples are created by drawing 
$M$ MNIST digits or CIFAR-10 images in an i.i.d. fashion and associating with them the sum or maximum of their corresponding gold labels. 
Regarding \textbf{VQAR} \citep{scallop}, 
the original benchmark includes a large number of queries that reduce \nesy~training to supervised one. 
To make training more challenging, we consider training samples associated with a large number of pre-images, averaging on more than 100 per training sample. To control the difficulty of training, we consider training samples whose logical formulae are of the form 
${name(superclass, o_1) \wedge rel(r,o_1,O_2)}$, 
where $superclass$ is the most generic class object $o_1$ can belong to according to the CRIC ontology, 
e.g., a toy is an object, and $a$ is the attribute object $O_2$ is associated with.   
The above formulae are slightly different than the ones described in Section~\ref{section:preliminaries}, as they include free variables: $o_1$ is a given object, while $O_2$ is a free one that can be found to any object within a given image. Our analysis applies without loss of generality to those settings due to the flexibility abduction gives us. 

\textbf{Engines.}
Like DeepProbLog, \scallop~relies on training using semantic loss \citep{pmlr-v80-xu18h}. 
However, it offers a scalable implementation of it.  
Research showed that \scallop~outperforms DeepProbLog, ABL \citep{ABL}, NeurASP \citep{neurasp} and the engine proposed in \citep{iclr2023} across a variety of tasks \citep{NeurIPS2023,scallop-journal}. 
In addition, \scallop~has state-of-the-art performance on MUGEN and VQAR, outperforming SDSC \citep{MUGEN} in MUGEN, and NMNs \citep{NMNs} and LXMERT \citep{lxmert} in VQAR. \dolphin~offers \nesy~training using losses based on fuzzy logic and has reported higher accuracy than \scallop~on a variety of benchmarks. 

\textbf{Encoders.}
We use a pretrained ResNet-18 convolutional neural network pretrained on ImageNet-1k for SUM-$M$, MAX-$M$, CIFAR-10, and HWF.
In VQAR, the object bounding boxes and features are obtained by passing the
images through pre-trained fixed-weight Mask RCNN and ResNet models.

\textbf{Computational environment.}
All experiments, except MUGEN, were performed on machines with two 20-core Intel Xeon Gold 6248 CPUs, four
NVIDIA GeForce RTX 2080 Ti (11 GB) GPUs, and 768 GB RAM.
MUGEN, as it required a larger GPU, was trained on an NVIDIA A100 40GB GPU.

\textbf{Additional implementation details.}
Across all experiments, we deal with directed proximity graphs and assume that there exists
a directed edge ${(\ell,x) \rightarrow (\ell',x')}$ only if $h(x')$ is in the top-$k$ neighborhood of $h(x)$, 
for $x,x' \in \+X$, for $k=1$.
We use a batch size of 64 for SUM-$M$ and MAX-$M$, with a learning rate of $1e-3$.
For CIFAR, we use an untrained ResNet18 model, with the same batch size and learning rate as MNIST.
For HWF, we used a batch size of 16 with a learning rate of $1e-4$.
For VQAR, we used a batch size of 512 with a learning rate of $1e-4$.
For MUGEN, we used a batch size of three, with a learning rate of $1e-4$.
Across all experiments, we used AdamW as the optimizer.
To compute proximity in the latent space, we used the FAISS open source library \citep{fais}.

\textbf{Software packages.} Our source code was implemented in Python 3.10.
We used the following python libraries: \texttt{scallopy}\footnote{\url{https://github.com/scallop-lang/scallop} (MIT license).}, 
\texttt{ISED}\footnote{\url{https://github.com/alaiasolkobreslin/ISED/tree/v1.0.0} (MIT license).}, 
\texttt{Dolphin}\footnote{\url{https://github.com/Dolphin-NeSy/Dolphin} (MIT license).},
\texttt{highspy}\footnote{\url{https://pypi.org/project/highspy/} (MIT license).}, \texttt{or-tools}\footnote{\url{https://developers.google.com/optimization/} (Apache-2.0 license).}, \texttt{PySDD}\footnote{\url{https://pypi.org/project/PySDD/}  (Apache-2.0 license).}, \texttt{PyTorch} and \texttt{PyTorch vision}. 

\textbf{Deep networks.} For MAX-$M$ and SUM-$M$, we used the MNIST CNN used in \citet{scallop}. 
For HWF, we used the CNN used in \citet{scallop-journal}. 
For VQAR and MUGEN, we used the same deep networks with \citet{scallop}.

\begin{table*}[t!]
    \centering
    \caption{Ablations for MNIST SUM-3 and $n=100$ for different pretrained and frozen encoders and batch sizes for \scallop.}
    \resizebox{\linewidth}{!}{
    \begin{tabular}{c|l|c|c|c|c|c}
    \toprule
         \makecell{\bf Batch size} & \textbf{Algorithms} & \makecell{\bf Accuracy} & \makecell{\bf \% Retained gold pre-images} & \makecell{\bf \% Gold pre-images} & \makecell{\bf Pruning time (s)} & \makecell{\bf \% of Training Time} \\
         \midrule
        \multirow{4}{*}{32} & \scallop &
        53.53 ± 17.51 & NA & NA & NA & NA \\
         & $~$ + \oursfrozen(\resnet-18) &
         41.68 ± 4.76 & 79.94 ± 2.07 & 92.43 ± 1.12 & 0.49 ± 0.01 & 5.76 ± 0.16 \\
         & $~$ + \oursfrozen(\resnet-152) &
         44.51 ± 10.65 & 78.84 ± 2.20 & 94.68 ± 1.90 & 0.61 ± 0.01 & 6.03 ± 0.17 \\
         & $~$ + \ourstrainable &
         \textbf{67.99 ± 12.37} & 79.23 ± 2.72 & 89.36 ± 2.26 & 0.49 ± 0.00 & 6.28 ± 0.33 \\
         \midrule
         \multirow{4}{*}{64} & \scallop &
         43.47 ± 14.02 & NA & NA & NA & NA \\
         & $~$ + \oursfrozen(\resnet-18) &
         \textbf{60.06 ± 17.52} & 80.83 ± 1.92 & 94.33 ± 1.76 & 0.50 ± 0.02 & 5.81 ± 0.08 \\
         & $~$ + \oursfrozen(\resnet-152) &
         49.61 ± 9.18 & 78.94 ± 2.99 & 90.35 ± 2.30 & 0.59 ± 0.01 & 6.12 ± 0.08 \\
         & $~$ + \ourstrainable &
         60.03 ± 21.72 & 79.18 ± 1.67 & 95.57 ± 1.12 & 0.49 ± 0.00 & 6.29 ± 0.35 \\
         \midrule
         \multirow{4}{*}{128} & \scallop &
         33.86 ± 16.94 & NA & NA & NA & NA \\
         & $~$ + \oursfrozen(\resnet-18) &
         39.81 ± 6.57 & 82.05 ± 1.46 & 95.49 ± 2.48 & 0.47 ± 0.02 & 6.17 ± 0.27 \\
         & $~$ + \oursfrozen(\resnet-152) &
         44.16 ± 16.59 & 79.60 ± 3.43 & 91.67 ± 2.05 & 0.55 ± 0.01 & 6.35 ± 0.06 \\
         & $~$ + \ourstrainable &
         \textbf{53.31 ± 16.86} & 77.88 ± 0.88 & 96.28 ± 0.45 & 0.48 ± 0.01 & 6.62 ± 0.13 \\
         \bottomrule
    \end{tabular}
    }
    \label{tab:abl_sum3_scallop}
\end{table*}    
\section{Ablations}
\label{appendix:ablations}

We provide the ablations for MNIST Sum-3, n=100, across different batch sizes and pretrained encoders for Scallop in Table~\ref{tab:abl_sum3_scallop}

\end{document}